\documentclass[letterpaper]{article} 
\usepackage{aaai24}  
\usepackage{times}  
\usepackage{helvet}  
\usepackage{courier}  
\usepackage[hyphens]{url}  
\usepackage{graphicx} 
\usepackage{natbib}  
\usepackage{caption} 
\usepackage{algorithm}

\usepackage{newfloat}
\usepackage{listings}
\DeclareCaptionStyle{ruled}{labelfont=normalfont,labelsep=colon,strut=off} 
\floatstyle{ruled}
\newfloat{listing}{tb}{lst}{}
\floatname{listing}{Listing}
\usepackage{color}
\usepackage{amsfonts} 
\usepackage{amsmath}
\usepackage[capitalise]{cleveref}
\usepackage{amsthm}
\usepackage{booktabs}
\usepackage{algpseudocode}
\newtheorem{theorem}{Theorem}
\newtheorem{definition}{Definition}

\DeclareMathOperator*{\argmax}{arg\,max}

\title{Mitigating Label Bias in Machine Learning: Fairness through Confident Learning}
\author{
    Yixuan Zhang\textsuperscript{\rm 1},
    Boyu Li\textsuperscript{\rm 2},
    Zenan Ling\textsuperscript{\rm 3},
    Feng Zhou\textsuperscript{\rm 4}\thanks{Corresponding author.}
}
\affiliations{
    \textsuperscript{\rm 1}China-Austria Belt and Road Joint Laboratory on AI and AM, Hangzhou Dianzi University, China\\
    \textsuperscript{\rm 2}Data Science Institute, University of Technology Sydney, Australia\\
    \textsuperscript{\rm 3}School of Electronic Information and Communications, Huazhong University of Science and Technology, China\\
    \textsuperscript{\rm 4}Center for Applied Statistics and School of Statistics, Renmin University of China, China\\

    yixuan.zhang@hdu.edu.cn, boyu.li@student.uts.edu.au, lingzenan@hust.edu.cn, feng.zhou@ruc.edu.cn
}

\usepackage{bibentry}

\begin{document}

\maketitle

\begin{abstract}
Discrimination can occur when the underlying unbiased labels are overwritten by an agent with potential bias, resulting in biased datasets that unfairly harm specific groups and cause classifiers to inherit these biases. In this paper, we demonstrate that despite only having access to the biased labels, it is possible to eliminate bias by filtering the fairest instances within the framework of confident learning. In the context of confident learning, low self-confidence usually indicates potential label errors; however, this is not always the case. Instances, particularly those from underrepresented groups, might exhibit low confidence scores for reasons other than labeling errors. To address this limitation, our approach employs truncation of the confidence score and extends the confidence interval of the probabilistic threshold. Additionally, we incorporate with co-teaching paradigm for providing a more robust and reliable selection of fair instances and effectively mitigating the adverse effects of biased labels. Through extensive experimentation and evaluation of various datasets, we demonstrate the efficacy of our approach in promoting fairness and reducing the impact of label bias in machine learning models. 
\end{abstract}

\section{Introduction}

In recent decades, we have observed a shift towards an AI-driven society, where machine learning techniques have profoundly affected various aspects of our lives, including finance~\cite{credit_example}, recruitment~\cite{company_example} and law~\cite{compas_juve}. When algorithmic decisions have a significant impact on our lives, it is crucial for decision-makers or regulators to have confidence in the algorithm's performance. While deep neural networks possess the capability to learn complex patterns from input data, they encounter a fundamental challenge in the data they learn from: the labels can be influenced by sensitive information, causing the neural networks to capture and reproduce these undesirable associations. Therefore, to ensure fairness in decision-making, it becomes essential to mitigate the impact of such undesirable relationships.

Research on training fair machine learning models has then received a lot of attention. To achieve fairness, one can incorporate fairness constraints into the learning objectives~\citep{fair_constraints,pmlr-v98-cotter19a,donini2020empirical,Rezaei_2020,pmlr-v119-roh20a}, or modify the model's predictions using threshold adjustments and calibration to align them with fairness constraints~\citep{equal_opportunity,multiaccuracy,inproceedings_bias_mitigate,petersen2021postprocessing}, or use data manipulation or representation learning methods before model training~\citep{NIPS2017_optimised_preprocessing,pmlr-v119-choi20a,pmlr-v108-jiang20a,kamiran_calders_2012}. Nonetheless, the aforementioned methods primarily focus on modifying the machine learning model to tackle the bias problem. There has been limited effort, as indicated by \citet{pmlr-v108-jiang20a}, in directly addressing the biased data itself, despite the fact that often the training data itself exhibits biased features and corresponding labels.

To combat the influence of label bias, numerous approaches have been proposed. For instance, \citet{pmlr-v108-jiang20a} addressed the biased data problem by formulating the label bias as a constrained optimization framework and using a re-weighting method to learn an equivalent unbiased labeling function. On the other hand, \citet{gpl} established fairness constraints by linking ground truth distribution and observed biased distribution. In this paper, we present another effective method that focuses on data selection. Intuitively, if we can select examples with labels less influenced by sensitive information, the network learned from such data will be more robust, leading to a fairer decision-making process.

The previous works aimed at mitigating label bias within the data itself using confidence scores generally encompass two phases~\citep{10.1613/confident_learning,CORDEIRO2023109013}. In the first phase, these approaches obtain confidence scores from the trained model and subsequently eliminate erroneous instances through a probabilistic threshold. The second phase involves model retraining using the remaining clean examples from the previous step to enhance model robustness. Nevertheless, when utilizing probability thresholds to determine the true labels, such methods heavily rely on self-confidence, of which instances with low self-confidence scores are often considered to contain label errors. However, this assumption does not always hold, particularly in scenarios of imbalanced data within the protected group. In these cases, individuals often face disadvantages and receive negative labels, causing deep networks to assign relatively low confidence scores to such instances due to their underrepresented nature~\citep{towards_fairer_datasets}.

To tackle the challenges mentioned earlier, we delve into the uncertainty of the sample selection process to combat biased labels. To alleviate the uncertainty linked with low-confidence examples, we suggest expanding the confidence intervals of the probabilistic threshold using a truncation function. Unlike previous methods that use the average confidence score to determine the probabilistic threshold, our approach reduces the impact of examples with extremely high confidence scores on the threshold. This adjustment facilitates the selection of examples with relatively lower confidence scores, which are often due to their association with disadvantaged groups, resulting in underrepresentation. Additionally, to further improve robustness, we employ the co-teaching paradigm, which involves training different models from subsets of data with varying demographic information. This allows us to cross-validate the selected fair instances from different demographic groups, enhancing the overall robustness and fairness of the approach. Our contributions are four-fold:
\begin{itemize}
    \item We introduce a data selection method that leverages confidence scores to tackle the issue of label bias. Notably, this approach is model-agnostic, thereby enabling its application with a wide range of models. This stands in contrast to many fairness-aware learning techniques aimed at addressing biased labels, which tend to be tightly integrated with specific models and the training process. 

    \item We present an extension of confidence intervals using a robust mean estimator, aimed at minimizing uncertainty in the process of data filtering. 
 
    \item We integrate the concept of co-teaching to enhance the robustness and reduce uncertainty by cross-validating selected fair instances originating from distinct demographic groups. 

    \item We assess the effectiveness of our approach across a range of benchmark datasets. The results highlight the benefits of our method in comparison to alternative baseline techniques.
\end{itemize}

\section{Preliminaries}
In this section, we briefly review the method addressing the biased label problem based on confidence measures. Let $X$ be the input variable, $Y$ be the observed output variable and $Z$ be the true labels. Consider a $k$-class classification problem, we denote the observed dataset as $(\mathbf{x},y)^N\in (\mathbb{R}^d,[k])^N$, where $d$ is the dimensionality of the input space and $N$ is the number of examples. In the context of confident learning, an instance with low self-confidence (as stated in \cref{def: self_confidence}) indicates a higher likelihood of being a label error~\citep{10.1613/confident_learning}. Unlike traditional confidence measures that focus on model predictions, leveraging this assumption, confident learning aims to identify examples with label errors by directly estimating the joint distribution between corrupted labels and true labels with a class-conditional noise process. To achieve this, the whole framework of confident learning combines principles of pruning noisy data, using probabilistic thresholds to estimate noise, and ranking examples to train with confidence.

\begin{definition}[Self-Confidence]
\label{def: self_confidence}
The predicted probability for some model $\theta$ that an instance $\mathbf{x}$ belongs to its given label $y=i$, i.e., $p(y=i; \mathbf{x},\mathbf{\theta})$.
\end{definition}
The confident learning process typically consists of three steps: estimation, pruning, and re-training. During the estimation process, we follow four steps to estimate the joint distribution between corrupted labels and true labels. Firstly, we compute the predicted probability of the $n$-th sample belonging to $j$, denoted as $\hat{p}_n^j = P(y_n=j; \mathbf{x}_n,\theta)$, where $j$ is the observed label and $j\in [k]$. We then use $t_j = \frac{1}{|\mathbf{X}_{y=j}|}\sum_{\mathbf{x}_n\in \mathbf{X}_{y=j}} \hat{p}_n^j$ for the $n$-th sample with observed label $j$ as the probabilistic threshold for class $j$. Next, for each $n$-th sample, we determine its true label $z_n$ is $\argmax_j \hat{p}_n^j$ and we have $\hat{p}_n^j>t_j$. Upon obtaining $z_n$, we keep a count of the label information in the count matrix $C_{y,z}$. For example, if $C_{y=1,z=0} = 40$, it means there are 40 examples that were labeled as $1$ but should have been labeled as $0$. To partition and count label errors, we then introduce the confident joint $\Bar{C}_{y=j,z=i}$, which is given by:
\begin{equation}
\label{eq: count_matrix}
    \Bar{C}_{y=j,z=i} = \frac{C_{y=j,z=i}}{\sum_{i\in [k]}C_{y=j,z=i}}\times |\mathbf{X}_{y=j}|.
\end{equation}
Using $\Bar{C}_{y=j,z=i}$ instead of $C_{y=j,z=i}$ offers the advantage of mitigating the sensitivity arising from class imbalance and distribution heterogeneity. This is achieved because $\Bar{C}_{y=j,z=i}$ employs per-class thresholding as a form of calibration~\citep{10.1613/confident_learning}. Then the joint distribution of $y$ and $z$ can be estimated using the confident joint, and the formula is expressed as:
\begin{equation}
    Q_{y=j,z=i} = \frac{\Bar{C}_{y=j,z=i}}{\sum_{j\in [k], i\in [k]}\Bar{C}_{y=j,z=i}}.
\end{equation}
The numerator ensures that the sum of $Q_{y=j,z=i}$ over all possible values of $i$ is calibrated to match the observed marginals for all $j$ in $[k]$. This calibration ensures that the row sums align with the observed distributions. On the other hand, the denominator calibrates the sum of $Q_{y=j,z=i}$ to be equal to 1, ensuring that the overall distribution is calibrated to sum up to 1. The estimations of $\Bar{C}_{y=j,z=i}$ and $Q_{y=j,z=i}$ are primarily aimed at facilitating the subsequent pruning of corrupted data. As the dataset size grows larger, this estimation method gradually approximates the true distribution more accurately~\citep{10.1613/confident_learning}. In other words, as we gather more data, the estimates become more reliable and better reflect the underlying relationships between the observed labels and the true labels, which enhances the effectiveness of identifying and handling corrupted data points.

In the pruning step, the objective is to identify and filter out erroneous examples. This can be achieved by various pruning approaches. For instance, we can directly remove the sets of examples that constitute the count in the off-diagonals (where true labels and observed labels are different) of $\Bar{C}_{y=j,z=i}$ and use the rest for training. Or we can employ $N \cdot Q_{y=j,z=i}$ to estimate the number of examples with label errors, which is denoted as $\Tilde{N}$, and then select top $\Tilde{N}$ examples based on the rank of predicted probability. Once the erroneous examples are filtered out, the class weights of the remaining examples can be readjusted, and the model can be retrained.

\section{Methodology}

In this section, we introduce our design for the data selection framework to eliminate label bias. We begin with a mathematical expression of the notions of label bias, and we then present the formulation of the data selection function.

\subsection{Notion of Label Bias}
Let $\mathcal{D}$ represent the true underlying distribution defined by the variables $(X, S, Z)$, and $\mathcal{\Tilde{D}}$ represent the observed corrupted distribution defined by $(X, S, Y)$. Here, $X$ refers to the non-protected attributes, $S$ represents the sensitive attributes, and $Z=1$ indicates the desirable outcome (typically considered the positive class). The privileged group is identified as $S=A$, while the disadvantaged group is labeled as $S=B$. In the subsequent sections, we denote the observed subset with membership $S=A$ as $D_A$, and the observed subset with membership $S=B$ as $D_B$. The entire observed dataset is denoted as $D$. We assume there exists a function $\mathcal{G}$ that involves flipping the labels of certain subsets of $D$ based solely on the values of $S$ and $Z$. Accordingly, the process of generating biased labels is represented as $Y=\mathcal{G}(X, S, Z)$. In line with previous work, we consider the scenario where negative examples in the privileged group might have been labeled as positive, while positive examples in the disadvantaged group could have been labeled as negative~\citep{unlocking_fairness, Dai2020LabelBL}. With this context, we introduce the following definitions:
\begin{equation*}
\begin{aligned}
\label{eq:label_bias_notion}
    \rho_A &= P(Y=1\mid S=A, Z=0),\\
    \rho_B &= P(Y=0\mid S=B, Z=1).
\end{aligned}
\end{equation*}
The above expressions demonstrate that negative instances from the privileged group $A$ are subject to flipping with a probability of $\rho_A$, while positive instances from the disadvantaged group $B$ are flipped with a probability of $\rho_B$. In \citet{osti_10190440}, the assumption is made that $\rho_A = 0$, meaning no flipping occurs for negative instances in the privileged group. On the other hand, in \citet{pmlr-v108-fogliato20a}, they set $\rho_B=0$, implying no flipping for positive instances in the disadvantaged group. In \citet{unlocking_fairness}, flipping occurs for both groups, and the flip rate is symmetric, i.e., $\rho_A = \rho_B \neq 0$.

\begin{figure}[!ht]
\centering
\includegraphics[width=1\linewidth]{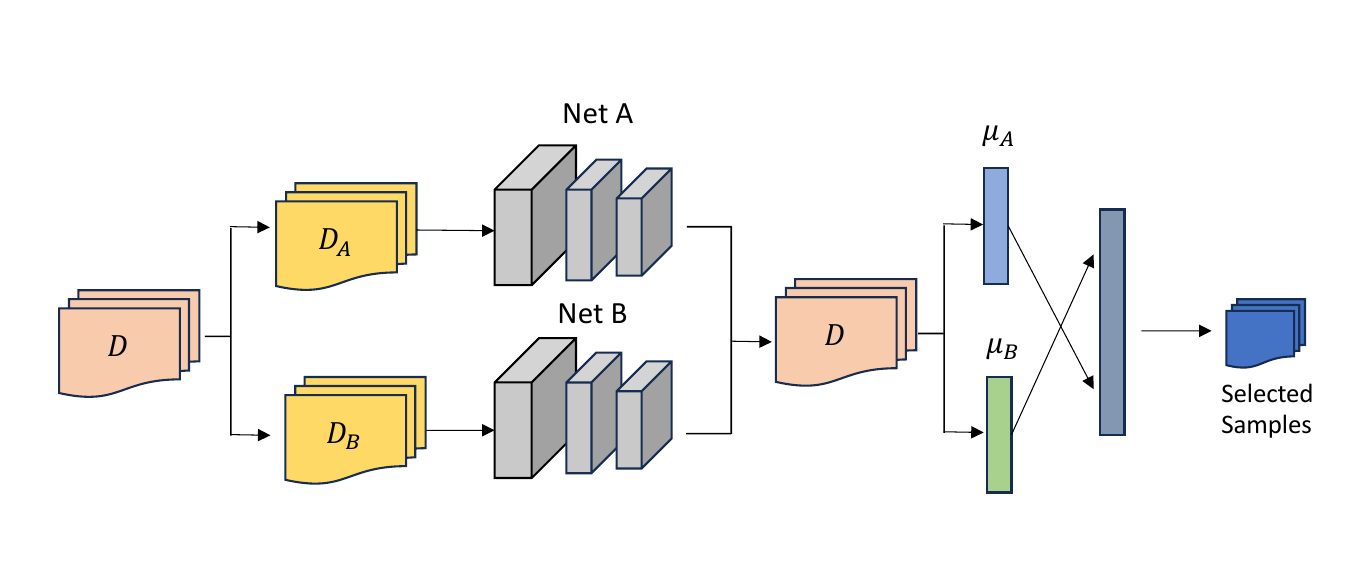}
\caption{Our method involves training two identical models on separate datasets, $D_A$ and $D_B$, and then evaluates the entire dataset $D$ to derive a probabilistic threshold, which guides the selection of the fairest instances. Though our approach is motivated by co-teaching, the training procedure diverges from co-teaching. During training, co-teaching evaluates the selected samples to help identify the most challenging instances. In contrast, we evaluate the entire dataset to validate that the selected samples are not affected by label bias.}
\label{fig:overall_process}
\end{figure}

\subsection{Selecting Unbiased Labels}
To improve the selection of unbiased examples, inspired by~\citet{NEURIPS2018_coteaching}, we train two models sharing the same architecture, denoted as $\mathbf{\theta}_A$ and $\mathbf{\theta}_B$, on separate subsets of the dataset: $D_A$ and $D_B$, respectively. Afterward, we evaluate each model on the entire dataset $D$ and obtain the confidence measure w.r.t. both $\theta_A$ and $\theta_B$. Following the definition of demographic parity, that predictions should be independent of sensitive attributes, we consider a model to be unbiased if its predictive performance, when trained on either $D_A$ or $D_B$, is the same from the results obtained on the entire dataset $D$. 
Then, to find the instances with label bias, we use the examples that lie in the off-diagonals of the confident joint in \cref{eq: count_matrix}. We denote the confident joint measured on $\theta_A$ as $\Bar{C}_A$ and the one measured on $\theta_B$ as $\Bar{C}_B$. $\Bar{C}_A$ and $\Bar{C}_B$ are constructed by the comparison of $t_j^A$ and $t_j^B$ where
\begin{equation}
\begin{aligned}
\label{eq:orig_selction_method}
    t_j^A &= \frac{1}{|\mathbf{X}_{y=j}|}\sum_{\mathbf{x}_n \in \mathbf{X}_{y=j}}\hat{p}(y_n = j; \mathbf{x}_n,\theta_A),\\
    t_j^B &= \frac{1}{|\mathbf{X}_{y=j}|}\sum_{ \mathbf{x}_n \in \mathbf{X}_{y=j}}\hat{p}(y_n = j ; \mathbf{x}_n,\theta_B).\\
\end{aligned}    
\end{equation}
Unlike conventional methods that identify biased instances assuming true labels are determined by:
\begin{equation*}
\label{eq: convention_z}
    z_n = \argmax_{j\in [k]}\hat{p}(y_n=j; \mathbf{x}_n,\theta),
\end{equation*}
we adopt the assumption that the true label is determined by: 
\begin{equation*}
\label{eq: new_z}
    \begin{aligned}
        z_n^A &= \argmax_{j\in [k]:\hat{p}(y_n=j; \mathbf{x}_n,\theta_A)\ge t^A_j}\hat{p}(y_n=j; \mathbf{x}_n,\theta_A), \\ 
        z^B_n &= \argmax_{j\in [k]:\hat{p}(y_n=j; \mathbf{x}_n,\theta_B)\ge t^B_j}\hat{p}(y_n=j; \mathbf{x}_n,\theta_B). \\
    \end{aligned} 
\end{equation*}
Consider the following example: if $D_A$ is primarily composed of positively labeled instances, while the majority of instances in $D_B$ are labeled negative, then for a given instance with observed positive label that measured on $\theta_A$, it is likely to have a higher $\hat{p}_n$ due to $\theta_A$ being overly confident about the positive class, resulting in a proportionally larger $t_j^A$, and vice versa. Therefore, by setting the probabilistic threshold as outlined in \cref{eq:orig_selction_method}, we can mitigate the class imbalance issue.

In addition to addressing the class imbalance problem, we account for the uncertainty associated with the probabilistic threshold, which arises from using the average of confidence scores to select instances belonging to a specific class. Such selection criteria can introduce biases, particularly for instances within the disadvantaged group. These instances may exhibit relatively lower confidence scores that are both beneath $t_j^A$ and $t_j^B$, resulting in their exclusion from the selection process. However, these instances may be correctly labeled and could be valuable for generalization. To provide these instances with a chance, following the approach in \citet{xia2022sample} and \citet{chen2020generalized}, we extend a classical M-estimator~\citep{catoni2011challenging} to truncate $t_j^A$ and $t_j^B$ as follows:
\begin{equation}
\begin{aligned}
    \Bar{t}_j^A &= \frac{1}{|\mathbf{X}_{y=j}|}\sum_{\mathbf{x}_n \in \mathbf{X}_{y=j}}\psi(\hat{p}(y_n = j; \mathbf{x}_n,\theta_A)),\\
    \Bar{t}_j^B &= \frac{1}{|\mathbf{X}_{y=j}|}\sum_{\mathbf{x}_n \in \mathbf{X}_{y=j}}\psi(\hat{p}(y_n = j; \mathbf{x}_n,\theta_B)),
\end{aligned}    
\end{equation}
where $\psi: \mathbb{R}\xrightarrow{} \mathbb{R}$ is a non-decreasing influence function, such that the widest possible choice is $\psi(x) = \log (1+x+\frac{|x|^2}{2})$ for $x\ge 0$. The selection of $\psi$ is motivated by the Taylor expansion of the exponential function, aiming to enhance the robustness of estimation results by mitigating the impact of extreme values. As a result, instances with extremely high confidence scores will have a limited effect on the probabilistic threshold, creating an opportunity for examples with relatively lower confidence scores that could still be correctly labeled to be included in the selection process. To achieve this, we derive the concentration inequalities for instances with low confidence scores based on 
\cref{thm:thm1}. The derivation of \cref{thm:thm1} follows the standard process in \citet{xia2022sample} and we include the proof in the Appendix for completeness. 
\begin{theorem}
\label{thm:thm1}
Consider an observation set $X_N = {x_1, \cdots, x_n}$ with mean $\mu$ and variance $\nu$. We utilize the non-decreasing influence function $\psi(x) = \log(1+x+\frac{x^2}{2})$. For any given $\epsilon > 0$, with a probability of at least $1-2\epsilon$, we have the following inequality:
\begin{equation*}
\left | \frac{1}{N}\sum^N_{n=1}\psi(x_n) - \mu \right | \le \frac{\nu(N+\frac{\nu\log(\epsilon^{-1})}{N^2})}{N-\nu}.
\end{equation*}
\end{theorem}
Let $N_s^A$ and $N_s^B$ represent the number of instances selected with respect to $\Bar{C}_A$ and $\Bar{C}_B$, respectively, where $N_s^A$ and $N_s^B$ are both less than or equal to the total number of instances $N$. Let's consider $\epsilon=\frac{1}{2N}$. Now, the threshold can be expressed as follows:
\begin{equation}
\begin{aligned}
\label{eq:selection_criteria}
    \mu_j^{A} = \Bar{t}_j^A - \frac{Q}{N_s^A - \nu},\\
    \mu_j^{B} = \Bar{t}_j^B - \frac{Q}{N_s^B - \nu},
\end{aligned}
\end{equation}
where $Q = \nu(N+\frac{\nu \log(2N)}{N^2})$. These inequalities help us understand the behavior of the estimation when dealing with less certain or less reliable instances.

\begin{algorithm}[]
  \caption{Training Algorithm}\label{alg:algorithm}
  \begin{algorithmic}[1]
   \State \textbf{Input:} training set $D$, model $\theta_A$, $\theta_B$ and $\theta$, epoch $T$, hyperparameter $N_s$ and $\nu$, loss function $\ell$, learning rate $\eta$.
    \For{\texttt{$t=0,\cdots,T$}}
    \State Obtain mini-batch $D^t$ from $D$ and split it into $D^t_A$ and $D^t_B$.
    \State $\theta_A \gets \theta_A - \eta \nabla \ell(D^t_A;\theta_B)$.
    \State $\theta_B \gets \theta_B - \eta \nabla \ell(D^t_B;\theta_B)$.
    \State Obtain $\mu_j^A$ and $\mu_j^B$ using \cref{eq:selection_criteria}
    \If{$\hat{p}(y_n=j;\mathbf{x}_n,\theta_A)\ge \mu_j^A$}
        \State $\hat{z}_n \gets \argmax_{j\in[k]}\hat{p}(y_n=j;\mathbf{x}_n,\theta_A)$
    \ElsIf{$\hat{p}(y_n=j;\mathbf{x}_n,\theta_B) \ge \mu_j^B$}
        \State $\hat{z}_n \gets \argmax_{j\in[k]}\hat{p}(y_n=j;\mathbf{x}_n,\theta_B)$
    \EndIf
    \State Compute the count joint using \cref{eq: count_matrix}.
    \State Fetch the example sets $\Tilde{D}^t_A$, $\Tilde{D}^t_B$ in the off-diagonal of the confident joint $\Bar{C}_A$ and $\Bar{C}_B$.
    \State $ D^t_S \gets D^t \text{\textbackslash} (\Tilde{D}^t_A \cup \Tilde{D}^t_B)$
    \State $\theta \gets \theta - \eta \nabla \ell (D^t_S;\theta)$
    \EndFor
  \end{algorithmic}
\end{algorithm}

\subsection{Training and Optimization}
The overall training process of the proposed approach is illustrated in \cref{fig:overall_process} and \cref{alg:algorithm}. Initially, two networks, denoted as $\theta_A$ and $\theta_B$, and sharing identical architectures, are established. Unlike the original co-teaching framework proposed in the previous literature~\citep{NEURIPS2018_coteaching}, $\theta_A$ and $\theta_B$ do not share the selected subset to update parameters; instead, they jointly evaluate the same dataset to identify biased data and enhance fairness. During each iteration, a mini-batch $D^t$ is drawn from the dataset $D$. This batch is subsequently divided into two separate batches based on the grouping factor $S$, yielding $D^t_A$ and $D^t_B$. We independently train $\theta_A$ and $\theta_B$ using $D^t_A$ and $D^t_B$, respectively. Following this, $D^t$ is evaluated using $\theta_A$ and $\theta_B$ to obtain $\mu_j^A$ and $\mu_j^B$, calculated via the selection criteria specified in \cref{eq:selection_criteria}. Through comparison with the acquired probabilistic threshold, instances residing in the off-diagonal elements of the joint count are identified. These instances are denoted as $\Tilde{D}^t_A$ and $\Tilde{D}^t_B$. The union set encompassing $\Tilde{D}^t_A$ and $\Tilde{D}^t_B$ is then removed, yielding the selected examples for model training. In experiments, we set $N^A_s = N^B_s = N_s$. To find the optimal values for $N_s$ and $\nu$, hyperparameter tuning on the validation set is employed.

\section{Experiments}
In this section, we demonstrate the effectiveness of our methods by comparing them with several baseline models on the benchmark datasets.

\subsection{Datasets} We use the synthetic data for verification and four sets of real-world data for comparison.

\textbf{Synthetic} We use the same setting for synthetic data generation as described in the work of \citet{2016_zafar}. We generate 95,750 fair examples with 2-dimensional non-sensitive attribute space and a 1-dimensional sensitive attribute space.

\textbf{Adult} \citep{adultcensus} The objective of this dataset is to predict whether a person's income exceeds \$50k per year. We consider two demographic groups based on gender. 

\textbf{ProPublica COMPAS} \citep{compas} This dataset contains information about criminal justice. The task is to predict recidivism based on various factors. We consider two demographics based on race.

\textbf{Credit Loan Data} \citep{misc_default_of_credit_card_clients_350} The dataset comprises credit card default records for 30,000 applicants from April to September 2005. We consider gender as the target demographic information.

\textbf{Law School Admissions} \citep{alma991020075739704336} The objective of this dataset is to predict whether or not a student will pass the bar. We form two demographic groups based on gender and use the pass bar as the ground-truth label.

\begin{table}[H]
    \centering
    \begin{tabular}{c c c c}
    \toprule
      Dataset  & N & $|S_A|$ &$|S_B|$ \\
      \midrule
      Synthetic & 95,750 & 72,025 & 23,725 \\
       Adult  & 46,032 & 31,113 & 14,919\\
       COMPAS & 7,214 & 3,518 & 3,696\\
       Credit Loan Data & 30,000 & 11,888 & 18,112 \\
       Law & 20,798 & 17,491 & 3,307\\
    \bottomrule
    \end{tabular}
    \caption{Dataset Description}
    \label{tab:my_label}
\end{table}

\subsection{Baselines}
For all the methods, we construct a simple neural network using ReLU activation functions. Our method is evaluated against several baselines, including \textbf{confident learning (CL)}~\citep{10.1613/confident_learning}, \textbf{LongReMix}~\citep{CORDEIRO2023109013}, \textbf{label bias correction (LC)}~\citep{pmlr-v108-jiang20a}, and the \textbf{group peer loss (GPL)} method as described in \citet{gpl}. Consistent hyperparameters are maintained across all experiments for all methods. Additional implementation details can be found in the Appendix.

\subsection{Fairness Violation}
We assess our performance on various datasets and methods with respect to diverse fairness metrics, which encompass (1) the demographic parity distance metric, (2) the difference of equal opportunity (DEO), and (3) the p\%. 

\textbf{Demographic parity distance metric}~\citep{DBLP:flex}: The definition of demographic parity is that the rate of positive predictions for $S=A$ should be equivalent to that for $S=B$. This metric is formulated as $|\mathbb{E}(\hat{Y}=1\mid S=A) - \mathbb{E}(\hat{Y}=1\mid S=B)|$.

\textbf{Difference of equal opportunity (DEO)}~\citep{equal_opportunity}: The concept of equal opportunity states that the true positive rates for $S=A$ should be identical to those for $S=B$. The difference can be quantified as: $|P(\hat{Y}=1 \mid S=A, Y=1) - P(\hat{Y}=1 \mid S=B, Y=1)|$.

\textbf{p\%}~\citep{biddle}: This measure closely resembles the demographic parity distance metric and can be formulated as: $\text{min}(\frac{P(\hat{Y}=1\mid S=A)}{P(\hat{Y}=1\mid S=B)}, \frac{P(\hat{Y}=1\mid S=A)}{P(\hat{Y}=1\mid S=B)})$.

Due to page limit, we present the DEO as our fairness violation metric in the table, while the outcomes pertaining to other fairness violation metrics like the demographic parity distance metric and the p\% are provided in the Appendix.

\subsection{Generating Biased Labels} 
We address two distinct categories of label bias: (1) symmetric bias, as defined by \citet{unlocking_fairness}, and (2) asymmetric bias, as outlined in \citet{osti_10190440}. For the symmetric bias case, we configure $\rho_A=\rho_B$ with values chosen from the set \{20\%, 40\%\}. Regarding asymmetric bias, we set $\rho_A \neq \rho_B$, where $\rho_A=0$ and $\rho_B$ is selected from \{20\%, 40\%\}. To ensure robust results, we perform 10 rounds of random shuffling on the training set, while retaining 10\% of the biased training examples as a validation set for hyperparameter optimization.

\begin{table*}[t]
\centering
\resizebox{0.95\linewidth}{!}{
\begin{tabular*}{\textwidth}{@{\extracolsep{\fill}}*{9}{c}}
 \toprule
  & \multicolumn{4}{c}{Adult} & \multicolumn{4}{c}{Compas} \\
\cmidrule{2-5}  \cmidrule{6-9} 
& \multicolumn{2}{c}{20\%} & \multicolumn{2}{c}{40\%} & \multicolumn{2}{c}{20\%} & \multicolumn{2}{c}{40\%} \\
\midrule
 Metric  & Err.(\%)($\downarrow$) & Vio.($\downarrow$) & Err.(\%)($\downarrow$) & Vio.($\downarrow$) & Err.(\%)($\downarrow$) & Vio.($\downarrow$) & Err.(\%)($\downarrow$) & Vio.($\downarrow$)\\
 \midrule
CL & 22.61$_{\pm 0.79}$ & 0.17$_{\pm 0.01}$ & 25.13$_{\pm 1.30}$ & 0.24$_{\pm 0.04}$ & 32.69$_{\pm 1.64}$ & 0.20$_{\pm 0.02}$ & 49.75$_{\pm 7.29}$ & $0.19_{\pm 0.02}$\\
LongReMix & $19.66_{\pm 3.09}$ & $0.18_{\pm 0.06}$ & $21.03_{\pm 1.21}$ & $0.19_{\pm 0.03}$ & $40.89_{\pm 2.39}$ & \textbf{0.09$_{\pm 0.02}$} & $43.21_{\pm 1.61}$ & \textbf{0.09$_{\pm 0.02}$} \\
LC & 18.89$_{\pm 0.37}$ & 0.17$_{\pm 0.02}$ & 20.10$_{\pm 0.11}$ & 0.18$_{\pm 0.03}$ & 34.21$_{\pm 0.42}$ & 0.15$_{\pm 0.03}$ & 35.19$_{\pm 0.23}$ & 0.16$_{\pm 0.01}$\\
GPL & $23.50_{\pm 2.40}$ & $0.18_{\pm 0.05}$ & $24.25_{\pm 2.36}$ & $0.16_{\pm 0.02}$& $44.88_{\pm 3.18}$ & $0.18_{\pm 0.03}$ & $46.81_{\pm 3.01}$ & $0.17_{\pm 0.02}$\\
\midrule
Ours & \textbf{15.75$_{\pm 0.62}$} & \textbf{0.12$_{\pm 0.02}$} & \textbf{16.83$_{\pm 1.48}$} & \textbf{0.15$_{\pm 0.02}$} & \textbf{30.75$_{\pm 0.25}$} & 0.15$_{\pm 0.02}$ & \textbf{31.86$_{\pm 1.61}$} & 0.17$_{\pm 0.02}$\\
\midrule
 & \multicolumn{4}{c}{Law} & \multicolumn{4}{c}{Credit} \\
\cmidrule{2-5}  \cmidrule{6-9} 
 & \multicolumn{2}{c}{20\%} & \multicolumn{2}{c}{40\%} & \multicolumn{2}{c}{20\%} & \multicolumn{2}{c}{40\%} \\
\midrule
 Metric  & Err.($\downarrow$) & Vio.($\downarrow$) & Err.($\downarrow$) & Vio.($\downarrow$) & Err.($\downarrow$) & Vio.($\downarrow$) & Err.($\downarrow$) & Vio.($\downarrow$)\\
 \midrule
CL & 15.49$_{\pm 3.56}$ & 0.28$_{\pm 0.05}$  & 16.01$_{\pm 2.15}$& 0.26$_{\pm 0.03}$ & 21.10$_{\pm 1.02}$ & 0.04$_{\pm 0.02}$ & 20.50$_{\pm 0.03}$ &0.03$_{\pm 0.01}$\\
LongReMix & $10.25_{\pm 1.40}$ & $0.13_{\pm 0.09}$ & $9.55_{\pm 0.45}$ & $0.11_{\pm 0.02}$ & $20.10_{\pm 0.99}$  & $0.03_{\pm 0.02}$  &$21.97_{\pm 1.99}$ & $0.04_{\pm 0.01}$ \\
LC & 9.54$_{\pm 0.07}$& \textbf{0.05$_{\pm 0.01}$} & 9.67$_{\pm 0.05}$& 0.06$_{\pm 0.01}$ &19.58$_{\pm 0.27}$ & 0.02$_{\pm 0.01}$& 20.49$_{\pm 0.11}$ & \textbf{0.02$_{\pm 0.01}$} \\
GPL & $10.44_{\pm 0.56}$ & $0.06_{\pm 0.01}$ & $10.26_{\pm 0.46}$ & \textbf{0.05$_{\pm 0.01}$}& $21.70_{\pm 1.06}$ & $0.02_{\pm 0.01}$& $22.82_{\pm 1.03}$ & $0.03_{\pm 0.02}$ \\
\midrule
Ours & \textbf{9.09$_{\pm 0.50}$} & \textbf{0.05$_{\pm 0.01}$} & \textbf{9.31$_{\pm 0.30}$} & 0.09$_{\pm 0.02}$& \textbf{19.13$_{\pm 0.77}$} & \textbf{0.01$_{\pm 0.01}$} & \textbf{19.33$_{\pm 0.68}$} & \textbf{0.02$_{\pm 0.01}$}\\
\bottomrule
\end{tabular*}}
\caption{Experiment Results (Symmetric Bias Scenario): Each row pertains to a specific method. The table illustrates the test errors (\%) and fairness violations of Confident Learning (CL), LongReMix, Label Correction methods (LC), Group Peer Loss (GPL), and our proposed approach.}
\label{tab:comparisons_results_sys}
\end{table*}

\subsection{Mean \emph{v.s.} Truncation}
We utilize the influence function $\psi(x)$ as a truncation mechanism for the probabilistic threshold, deviating from the original averaging-based approach. This modification allows for a meaningful comparison between the two metrics. Synthetic data is employed for assessment. We label the methods using \cref{eq:orig_selction_method} as ``M" with hyperparameters $N_s=0.6$ and $\nu = 10^{-2}$ fixed and we label the method using \cref{eq:selection_criteria} as ``T". Analysis of \cref{tab:syn_compare} reveals that adopting the truncated probabilistic threshold for data selection outperforms the conventional mean estimator approach, both under symmetric and asymmetric bias. Discrepancies between ``M" and ``T" widen with increasing bias, leading to a 3.60\% error for ``M", while the increase is marginal for the ``T" under symmetric bias. In asymmetric bias, ``T" still outperforms ``M", despite higher overall accuracy. Minimal disparity is observed in fairness violation between the two methods on synthetic data. The truncated method consistently exhibits lower fairness violations, often achieving zero violations.

\subsection{Comparison Results}
We present the results in \cref{tab:comparisons_results_sys} and \cref{tab:comparisons_results_asy}. It is evident that our approach consistently produces fair classifiers, often achieving the lowest accuracy errors and fairness violations across all methods on the four real-world datasets. \cref{tab:comparisons_results_sys} presents the results under the symmetric bias setting, where we observe that the test errors increase when the bias magnitude is 40\% compared to when it's 20\%. However, our method exhibits only a marginal increase in test errors and fairness violations with an elevated bias level, indicating more stable performance. It is important to highlight that while LongReMix exhibits the lowest fairness violation (measured by DEO) on the compas dataset, its test error rate is high. This outcome stems from LongRemix predominantly predicting negative outcomes for the majority of examples, resulting in a minimal DEO. However, upon reviewing the outcomes as presented in the Appendix, it becomes evident that when evaluated using alternate fairness metrics, the fairness violations are notably high. Comparing the outcomes in \cref{tab:comparisons_results_asy}, it is apparent that all the methods have better performance than the results displayed in \cref{tab:comparisons_results_sys}. Despite this, our method maintains superior or comparable predictive accuracy when contrasted with other approaches, and it consistently exhibits the lowest levels of fairness violations. An interesting observation is that the results of confident learning itself do not consistently yield fair classifiers compared to label bias correction methods for fairness. This is due to its reliance on class-dependent noise, leading to higher test errors and fairness violations. The reason for this lies in the fact that confident learning does not incorporate demographic information, unlike the label bias correction methods, resulting in less effective performance.

\subsection{Hyperparameter Analysis}
We explore the impact of hyperparameters $\nu$ in the range of $\{10^{-4},10^{-3},10^{-2},10^{-1}\}$ and $N_s$ in the range of $\{0.5,0.6,0.7,0.8,0.9\}$ to examine their impact using synthetic data. The experiment's focus is on assessing the influence of these two hyperparameters. In \cref{fig:syn_illus_s_ns}, we conduct an experiment by keeping $N_s = 0.75$ fixed and varying the value of $\nu$ from $10^{-4}$ to $10^{-1}$. \cref{fig:syn_illus_s_ns} demonstrates that the overall test error when $\rho_A = \rho_B = 0.4$ is higher compared to the error rate when $\rho_A = \rho_B = 0.2$ as we change $N_s$ and $\nu$. We do not report the fairness violation since they are overall very small (close to 0) and consistently exhibit no apparent variation when we change the value of $\nu$ and $N_s$. Regarding $N_s$, we hold $\nu=10^{-2}$ constant and change the value from 0.5 to 0.9. The graphical representation of how $N_s$ and $\nu$ influence the probabilistic threshold is presented in \cref{fig:syn_illus_s_ns}. The plot reveals that the influence function smooths higher confidence scores. Additionally, both $N_s$ and $\nu$ play a role in regulating the deviation from the initial value. Larger values of $\nu$ result in greater deviations from the original value, while smaller values of $N_s$ lead to more substantial deviations from the original value.

\begin{table}[ht]
    \centering
    \resizebox{0.95\linewidth}{!}{
    \begin{tabular}{c c c c c}
    \toprule
    & \multicolumn{4}{c}{Symmetric}\\
    \cmidrule{2-5}
        & \multicolumn{2}{c}{20\%} & \multicolumn{2}{c}{40\%} \\
      \midrule
       & Err.(\%) & Vio. & Err.(\%) & Vio. \\
       \midrule
       M & $1.51_{\pm 0.26}$ & $0.01_{\pm 0.01}$ & $3.60_{\pm 1.59}$ & $0.02_{\pm 0.01}$\\
      T & \textbf{0.77$_{\pm 0.17}$} & \textbf{0.00$_{\pm 0.00}$} & \textbf{0.91$_{\pm 0.37}$} & \textbf{0.01$_{\pm 0.00}$} \\
      \midrule
      & \multicolumn{4}{c}{Asymmetric} \\
      \cmidrule{2-5}
        & \multicolumn{2}{c}{20\%} & \multicolumn{2}{c}{40\%} \\
        \midrule
      & Err.(\%) & Vio. & Err.(\%) & Vio. \\
       \midrule
       M & $0.45_{\pm 0.08}$ & $0.01_{\pm 0.01}$ & $0.55_{\pm 0.36}$ & $0.01_{\pm 0.01}$\\
      T & \textbf{0.36$_{\pm 0.06}$} & \textbf{0.00$_{\pm 0.00}$} & \textbf{0.38$_{\pm 0.10}$} & \textbf{0.00$_{\pm 0.00}$} \\
    \bottomrule
    \end{tabular}}
    \caption{Comparisons between the mean estimator and the truncation method under symmetric and asymmetric bias.}
    \label{tab:syn_compare}
\end{table}

\begin{table*}[t]
\centering
\resizebox{0.95\linewidth}{!}{
\begin{tabular*}{\textwidth}{@{\extracolsep{\fill}}*{9}{c}}
 \toprule
  & \multicolumn{4}{c}{Adult} & \multicolumn{4}{c}{Compas} \\
\cline{2-5}  \cline{6-9} 
& \multicolumn{2}{c}{20\%} & \multicolumn{2}{c}{40\%} & \multicolumn{2}{c}{20\%} & \multicolumn{2}{c}{40\%} \\
\midrule
 Metric  & Err.(\%)($\downarrow$) & Vio.($\downarrow$) & Err.(\%)($\downarrow$) & Vio.($\downarrow$) & Err.(\%)($\downarrow$) & Vio.($\downarrow$) & Err.(\%)($\downarrow$) & Vio.($\downarrow$)\\
 \midrule
CL & $19.70_{\pm 0.49}$ & $0.14_{\pm 0.01}$ & $21.48_{\pm 0.76}$ & $0.16_{\pm 0.01}$ & $33.00_{\pm 0.20}$ & $0.19_{\pm 0.01}$ & $32.08_{\pm 0.74}$ & $0.19_{\pm 0.01}$ \\
LongReMix & $19.86_{\pm 1.33}$ & $0.15_{\pm 0.06}$ & $18.84_{\pm 0.98}$ & $0.16_{\pm 0.01}$ & $33.35_{\pm 1.01}$ & $0.18_{\pm 0.07}$ & $42.80_{\pm 0.83}$ & $0.15_{\pm 0.03}$\\
LC & \textbf{16.36$_{\pm 0.05}$} & $0.13_{\pm 0.01}$ & $17.53_{\pm 0.12}$& $0.10_{\pm 0.02}$ & $37.20_{\pm 0.23}$ & $0.11_{\pm 0.01}$ & $41.16_{\pm 0.20}$& \textbf{0.04$_{\pm 0.00}$} \\
GPL & $18.87_{\pm 0.99}$ & \textbf{0.07$_{\pm 0.02}$} & $18.72_{\pm 0.64}$ & $0.11_{\pm 0.02}$ & $39.28_{\pm 1.10}$ & \textbf{0.10$_{\pm 0.02}$} & $41.44_{\pm 1.11}$ & $0.08_{\pm 0.01}$ \\
\midrule
Ours & $17.76_{\pm 0.36}$& \textbf{0.07$_{\pm 0.01}$} & \textbf{17.46$_{\pm 0.66}$} & \textbf{0.09$_{\pm 0.01}$} & \textbf{30.75$_{\pm 0.20}$} & $0.14_{\pm 0.02}$& \textbf{30.06$_{\pm 0.17}$} & $0.16_{\pm 0.01}$\\
\midrule
 & \multicolumn{4}{c}{Law} & \multicolumn{4}{c}{Credit} \\
\cline{2-5}  \cline{6-9} 
 & \multicolumn{2}{c}{20\%} & \multicolumn{2}{c}{40\%} & \multicolumn{2}{c}{20\%} & \multicolumn{2}{c}{40\%} \\
\midrule
Metric  & Err.(\%)($\downarrow$) & Vio.($\downarrow$) & Err.(\%)($\downarrow$) & Vio.($\downarrow$) & Err.(\%)($\downarrow$) & Vio.($\downarrow$) & Err.(\%)($\downarrow$) & Vio.($\downarrow$)\\
 \midrule
CL & $18.12_{\pm 4.16}$ & $0.31_{\pm 0.05}$ & $22.32_{\pm 6.01}$ & $0.31_{\pm 0.03}$ & $20.43_{\pm 1.06}$ & $0.02_{\pm 0.02}$ & $20.27_{\pm 0.96}$ & $0.02_{\pm 0.01}$\\
LongReMix & $9.61_{\pm 0.51}$ & $0.13_{\pm 0.04}$ & $10.15_{\pm 0.80}$ & $0.14_{\pm 0.08}$ & $21.41_{\pm 0.45}$ & $0.03_{\pm 0.01}$ & $21.46_{\pm 0.40}$ & $0.02_{\pm 0.01}$\\
LC &  $9.33_{\pm 0.07}$ & \textbf{0.01$_{\pm 0.00}$} & $9.39_{\pm 0.07}$& \textbf{0.01$_{\pm 0.00}$}& $20.43_{\pm 0.09}$ & $0.02_{\pm 0.00}$ & $21.39_{\pm 0.05}$ & $0.02_{\pm 0.00}$\\
GPL & $10.69_{\pm 0.50}$ & $0.05_{\pm 0.01}$& $11.00_{\pm 0.60}$ & $0.06_{\pm 0.01}$& $21.18_{\pm 0.56}$ & $0.03_{\pm 0.01}$ & $20.80_{\pm 0.61}$ & $0.02_{\pm 0.01}$\\
\midrule
Ours & \textbf{8.99$_{\pm 0.53}$} & $0.07_{\pm 0.03}$ & \textbf{9.23$_{\pm 0.75}$}& $0.08_{\pm 0.01}$ &  \textbf{18.77$_{\pm 0.21}$} & \textbf{0.02$_{\pm 0.01}$} & \textbf{18.93$_{\pm 0.39}$} & \textbf{0.01$_{\pm 0.00}$}  \\
\bottomrule
\end{tabular*}}
\caption{Experiment Results (Asymmetric Bias Scenario): Each row pertains to a specific method. The table illustrates the test errors (\%) and fairness violations of Confident Learning (CL), LongReMix, Label Correction methods (LC), Group Peer Loss (GPL), and our proposed approach.}
\label{tab:comparisons_results_asy}
\end{table*}

\begin{figure*}[ht]
    \centering
        \includegraphics[width=1\linewidth]{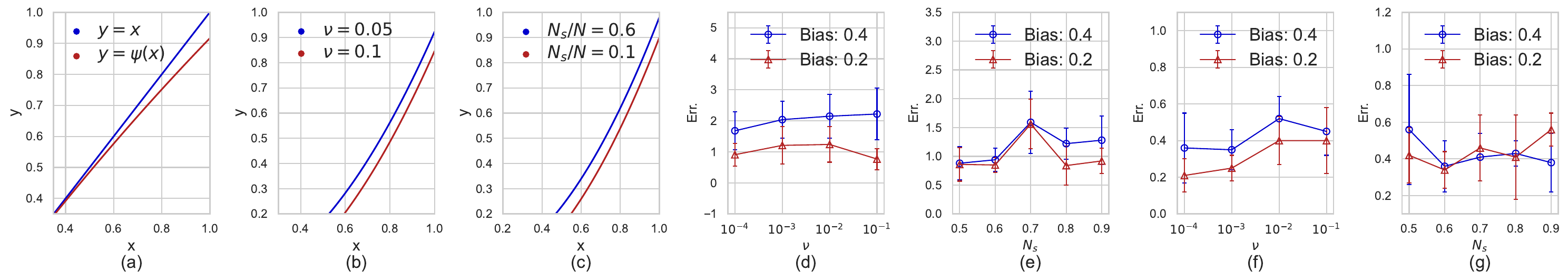}
    \caption{The illustration of the impact of $\psi(x)$: Plot (a) - (c), for (b), we set $\frac{N_s}{\nu} = 0.65$, for (c), we set $\nu = 0.5$; Influence of Hyperparameters on Probabilistic Threshold: Plot (d) - (g). Plot (d) and (e) are under the symmetric bias setting, while Plot (f) and (g) are under the asymmetric bias setting. The impact of hyperparameters $\nu$ and $N_s$ is investigated using synthetic data. }
    \label{fig:syn_illus_s_ns}
\end{figure*}

\section{Related Work}
In this study, we introduced a data selection framework to mitigate label bias and achieve fairness. Existing data selection techniques \citep{Loshchilov2017DecoupledWD, Kawaguchi2019OrderedSA, DBLP:journals/corr/jiang, DBLP:journals/corr/coleman, DBLP:journals/corr/LoshchilovH15,Killamsetty2020GLISTERGB, pmlr-v162-mindermann22a} successfully address biased labels based on loss, but they are not well-suited for fairness. Our approach is loss-agnostic and specifically designed for fairness. Though FairBatch \citep{Roh0WS21} introduces adaptive data selection to achieve fairness, the entire framework neglects label bias. Regarding the research on mitigating biased label problems in fairness, \citet{pmlr-v108-fogliato20a} present experimental findings that even minor biases in observed labels can cast uncertainty on the reliability of analyses that depend on the biased outcomes. \citet{osti_10190440} explored the impact of biased training data on classification and revealed the potential to recover the Bayes Optimal Classifier by combining Empirical Risk Minimization with the equal opportunity constraint under various bias models. Building on these studies, \citet{Dai2020LabelBL} introduced a comprehensive framework addressing label bias and data distribution shift for fairness. \citet{gpl} introduced bias using group-dependent label noise and proposed a surrogate loss function to handle various fairness constraints in the presence of corrupted data with biased labels. Notably, they incorporated insights from Peer loss \citep{liu2020peer} to formulate the group peer loss, addressing group-dependent label noise. Conversely, \citet{pmlr-v108-jiang20a} approached label bias through a constrained optimization framework, proposing a re-weighting method to correct instances affected by label bias. In contrast to previous approaches modifying the objective function during training, our focus is on data selection for explicit label bias mitigation.

\section{Conclusion}
This paper demonstrates that bias can be eliminated by implementing confident learning and filtering the fairest instances, even with access to biased labels. However, challenges arise with instances of low self-confidence, especially for underrepresented groups, which may still contain biased labels and remain unselected. To address this, the proposed approach utilizes the lower bound of confidence intervals, significantly enhancing the robustness and reliability of confidence score thresholding compared to traditional mean estimation methods. Leveraging interval estimation effectively mitigates the adverse effects of biased labels, promoting fairness in machine learning models. Furthermore, incorporating co-teaching enhances fair instance selection, contributing to the overall effectiveness of the framework. Extensive experimentation and evaluation on various datasets underscore the proposed method's efficacy in reducing the impact of label bias and promoting fairness.

\appendix
\setcounter{theorem}{0}
\label{sec:reference_examples}
\section*{Appendix}
\label{sec:appendix_proof}

\subsection*{Proof of Theorem}
\begin{theorem}
Consider an observation set $X_N = {x_1, \cdots, x_n}$ with mean $\mu$ and variance $\nu$. We utilize the non-decreasing influence function $\psi(x) = \log(1+x+\frac{x^2}{2})$. For any given $\epsilon > 0$, with a probability of at least $1-2\epsilon$, we have the following inequality:
\begin{equation*}
\left | \frac{1}{N}\sum^N_{n=1}\psi(x_n) - \mu \right | \le \frac{\nu(N+\frac{\nu\log(\epsilon^{-1})}{N^2})}{N-\nu}.
\end{equation*}
\end{theorem}

\begin{proof}
The proof is standard and has appeared in previous works~\citep{xia2022sample,chen2020generalized}. For completeness, we include the proof below. Consider a variable $x$ with truncation denoted as $t = \frac{1}{N}\sum^n_{i=1}\psi(x_n)$. We establish $r(t) = \sum^N_{i=1}\psi(\beta(x_n-t))=0$, where $\beta$ is a strictly positive real parameter. Suppose for $\theta$, we define the quantity $r(\theta) = \frac{1}{\beta N}\sum^{N}_{n=1}\psi(\alpha (x_n-\theta))$, with $r(\theta)$ as a non-increasing random variable due to the non-decreasing nature of $\psi$. The choice of $\psi$ function is inspired by Taylor-like expansion introduced in \cite{chen2018steins}, the intuition behind that is to generalize the original M-estimator to accommodate scenarios where samples can have finite $\alpha$-th moment with $\alpha \in (1,2)$. 

\textbf{The proof of the main inequalities.} In accordance with Lemma 2.1 in \citet{chen2020generalized}, and setting $\alpha=2$, we can obtain the following inequalities:
$\exp {\beta N r(\theta)}\le \exp {\beta N (t-\theta)+\beta^2 N (\nu+(t-\theta)^2)} $, and $\exp {-\beta N r(\theta)}\le \exp {-\beta N (t-\theta)+\beta^2 N (\nu +(t-\theta)^2)}$. These lead us to the subsequent bounds:
\begin{equation}
\begin{aligned}
    B_{-}(\theta) = t - \theta - \beta [\nu +(t-\theta)^2]-\frac{\log (\epsilon^{-1})}{\beta N}, \\
    B_{+}(\theta) = t - \theta + \beta [\nu +(t-\theta)^2]+\frac{\log (\epsilon^{-1})}{\beta N}. 
\end{aligned}
\end{equation}
Applying Lemma 2.2 presented in \citet{chen2020generalized}, we subsequently deduce that $P(r(\theta)>B_{-}(\theta))\ge 1-\epsilon$ and $P(r(\theta)<B_{+}(\theta))\ge 1-\epsilon$ hold true for any $\theta \in \mathbb{R}$. Consider $t_a$ as the largest root of the quadratic equation $B_{-}(\theta)$ and $t_b$ as the smallest root of the quadratic $B_{+}(\theta)$. Referring to findings in \citet{chen2020generalized}, we then arrive at the following:
\begin{equation}
    \begin{aligned}
        t_a - t &\ge - \frac{\beta \nu + \frac{\log (\epsilon^{-1})}{\beta N}}{\beta -1},\\
        t_b - t &\le \frac{\beta \nu + \frac{\log (\epsilon^{-1})}{\beta N}}{\beta -1}.
    \end{aligned}
\end{equation}
With probability at least $1-2\epsilon $, we have $t_a\le t\le t_b$. We can choose $\beta = \frac{N}{\nu}$, and let $\epsilon = \frac{1}{2N}$ then we have
\begin{equation}
    \left | \frac{1}{N}\sum^n_{n=1}\psi(x_n) - t \right | \le \frac{\nu(N+\frac{\nu\log(\epsilon^{-1})}{N^2})}{N-\nu}. 
\end{equation}
\end{proof}

\subsection*{Other Fairness Violation Notions}

The fairness violation outcomes for various evaluation metrics under a symmetric bias setting are provided in \cref{tab:other_fair_vio_20} and \cref{tab:other_fair_vio_40}. By combining these findings with the results outlined in \cref{tab:comparisons_results_sys}, we can conclude that our method consistently demonstrates the least fairness violation across diverse fairness metrics. Even when $\rho_A = \rho_B = 40\%$, our approach remains competitive, delivering notable results for datasets such as DP in the Adult and Compas datasets, as well as p\% and DP in the Credit datasets, despite not always achieving the lowest DP or the highest p\%.

\label{appendix: other_fair_vio}
\begin{table*}[h]
\centering
\resizebox{\linewidth}{!}{
\begin{tabular*}{\textwidth}{@{\extracolsep{\fill}}*{7}{c}}
 \toprule
Dataset & Metric & Ours  & CL & LongReMix &  LC & GPL  \\
\midrule
Adult & DP($\downarrow$) & \textbf{0.12$_{\pm 0.01}$} & $0.19_{\pm 0.01}$ & $0.16_{\pm 0.02}$ & $0.14_{\pm 0.01}$ & $0.13_{\pm 0.01}$ \\
      & p\%($\uparrow$) & \textbf{0.87$_{\pm 0.01}$} & $0.79_{\pm 0.01}$ & $0.83_{\pm 0.03}$ & $0.86_{\pm 0.01}$ & \textbf{0.87$_{\pm 0.01}$} \\
\midrule
Compas & DP($\downarrow$) &  \textbf{0.02$_{\pm 0.01}$} & $0.03_{\pm 0.01}$ & $0.13_{\pm 0.02}$ & $0.10_{\pm 0.01}$ & $0.04_{\pm 0.02}$\\
       & p\%($\uparrow$) & \textbf{0.97$_{\pm 0.01}$} & $0.96_{\pm 0.02}$ & $0.77_{\pm 0.06}$ & $0.85_{\pm 0.01}$ & $0.93_{\pm 0.04}$\\
\midrule
Law & DP($\downarrow$) & \textbf{0.15$_{\pm 0.01}$} & $0.23_{\pm 0.03}$ & $0.24_{\pm 0.02}$ & $0.24_{\pm 0.01}$ & $0.18_{\pm 0.02}$\\
    & p\%($\uparrow$) & \textbf{0.84$_{\pm 0.02}$} & $0.73_{\pm 0.04}$ & $0.70_{\pm 0.11}$ & $0.72_{\pm 0.06}$ & $0.81_{\pm 0.01}$\\
\midrule
Credit & DP($\downarrow$) & \textbf{0.01$_{\pm 0.01}$} & \textbf{0.01$_{\pm 0.01}$}  & \textbf{0.01$_{\pm 0.01}$}  & \textbf{0.01$_{\pm 0.01}$}  & $0.02_{\pm 0.00}$\\
       & p\%($\uparrow$) & $0.98_{\pm 0.01}$ & $0.98_{\pm 0.01}$ & \textbf{0.99$_{\pm 0.01}$} & ${0.98_{\pm 0.01}}$ & $0.98_{\pm 0.00}$  \\
\bottomrule
\end{tabular*}}
\caption{Experiment Results Regarding Other Measures of Fairness Violations (Symmetric Bias Scenario (20\%)): Each row pertains to a specific method. The table illustrates the fairness violations regarding p\% and the Demographic Parity distance metric (DP) of Confident Learning (CL), LongReMix, Label Correction methods (LC), Group Peer Loss (GPL), and our proposed approach. Notably, methods that achieve the highest p\% and lowest DP for each dataset under varying bias levels are highlighted using bold font.}
\label{tab:other_fair_vio_20}
\end{table*}

\begin{table*}[!ht]
\centering
\resizebox{\linewidth}{!}{
\begin{tabular*}{\textwidth}{@{\extracolsep{\fill}}*{7}{c}}
 \toprule
Dataset & Metric & Ours  & CL & LongReMix & LC & GPL  \\
\midrule
Adult & DP($\downarrow$) & 0.13$_{\pm 0.01}$ & $0.22_{\pm 0.01}$ & $0.17_{\pm 0.02}$ & $0.15_{\pm 0.02}$ & \textbf{0.12$_{\pm 0.01}$}\\
      & p\%($\uparrow$) & \textbf{0.86$_{\pm 0.01}$} & $0.76_{\pm 0.01}$ & $0.79_{\pm 0.05}$ & $0.85_{\pm 0.04}$ & $0.82_{\pm 0.02}$\\
\midrule
Compas & DP($\downarrow$)  & $0.05_{\pm 0.02}$ & $0.13_{\pm 0.01}$ & $0.14_{\pm 0.01}$ & $0.12_{\pm 0.02}$& \textbf{0.03$_{\pm 0.01}$}\\
       & p\%($\uparrow$) & \textbf{0.95$_{\pm 0.03}$} & $0.76_{\pm 0.04}$ & $0.74_{\pm 0.03}$ & $0.81_{\pm 0.02}$& $0.93_{\pm 0.03}$\\
\midrule
Law & DP($\downarrow$) & \textbf{0.16$_{\pm 0.03}$} & $0.21_{\pm 0.03}$ & $0.18_{\pm 0.00}$ & $0.17_{\pm 0.01}$ & $0.18_{\pm 0.02}$\\
    & p\%($\uparrow$) & \textbf{0.82$_{\pm 0.02}$} & $0.76_{\pm 0.03}$ & $0.80_{\pm 0.01}$ & \textbf{0.82$_{\pm 0.01}$} & $0.81_{\pm 0.01}$\\
\midrule
Credit & DP($\downarrow$) & $0.02_{\pm 0.01}$ &  $0.02_{\pm 0.00}$ & \textbf{0.01$_{\pm 0.00}$} & $0.02_{\pm 0.00}$ & $0.02_{\pm 0.01}$ \\
       & p\%($\uparrow$) & $0.98_{\pm 0.01}$ & $0.97_{\pm 0.02}$ & $0.97_{\pm 0.02}$ & \textbf{0.99$_{\pm 0.01}$} & $0.98_{\pm 0.01}$\\
\bottomrule
\end{tabular*}}
\caption{Experiment Results Regarding Other Measures of Fairness Violations (Symmetric Bias Scenario (40\%)): Each row pertains to a specific method. The table illustrates the fairness violations regarding p\% and the Demographic Parity distance metric (DP) of Confident Learning (CL), LongReMix, Label Correction methods (LC), Group Peer Loss (GPL), and our proposed approach. Notably, methods that achieve the highest p\% and lowest DP for each dataset under varying bias levels are highlighted using bold font.}
\label{tab:other_fair_vio_40}
\end{table*}

\subsection*{Implementation Details}
\label{sec:appendix_implementation}
We utilize a two-layer perceptron with a ReLU activation function in our approach, maintaining a consistent architecture across other baseline methods. Both $\theta_A$ and $\theta_B$ adopt this two-layer perceptron architecture. Similarly, the model trained on the selected fair instances also follows this architecture. In the case of the LongReMix baseline, which requires a subset of unlabeled data, we partition 30\% of the training set from each batch for this purpose.

\subsection*{Gnerating Biased Labels}

In this paper, we address two distinct categories of label bias: (1) symmetric bias, as defined by \citet{unlocking_fairness}, and (2) asymmetric bias, as outlined in \citet{osti_10190440}. For the symmetric bias case, we configure $\rho_A=\rho_B$, while in asymmetric bias, we set $\rho_A \neq \rho_B$, where 
\begin{equation*}
\begin{aligned}
    \rho_A &= P(Y=1\mid S=A, Z=0),\\
    \rho_B &= P(Y=0\mid S=B, Z=1).
\end{aligned}
\end{equation*}

For symmetric bias, the assumption posits that the likelihood of a privileged group receiving a positive outcome (which they shouldn't) equals the probability of a protected group receiving a negative outcome (when they should receive a positive outcome). To illustrate, envision a hiring process where a competent candidate from Group B should be hired but isn't, while simultaneously an unqualified candidate from Group A is hired. In this context, we assert that the outcomes for each unqualified individual in Group A and the qualified individual in Group B share an equal probability of being inverted.

For the asymmetric bias case, we incorporate the concept of Under-Representation Bias, as discussed in \citep{osti_10190440}, into our study. In this context, let's take the example of hiring, where the underlying decision function exhibits excessive negativity towards the disadvantaged group. More precisely, the training data lacks an adequate representation of positive examples from Group $B$, resulting in a diminished presence of such instances in the training dataset.

\section{Acknowledgments}
We thank the anonymous reviewers for their constructive feedback and suggestions on how to strengthen the paper. This work was supported by NSFC Project (No. 62106121), the MOE Project of Key Research Institute of Humanities and Social Sciences (22JJD110001), the fund for building world-class universities (disciplines) of Renmin University of China (No. KYGJC2023012), China-Austria Belt and Road Joint Laboratory on Artificial Intelligence and Advanced Manufacturing of Hangzhou Dianzi University, and the Public Computing Cloud, Renmin University of China. 

\bibliography{aaai24}

\end{document}